\documentclass{colt2020} 
\usepackage{times}
\usepackage[utf8]{inputenc} 
\usepackage{hyperref}       
\usepackage{url}            
\usepackage{booktabs}       
\usepackage{amsfonts}       

\usepackage{amssymb}
\usepackage{amsmath}
\usepackage{subcaption}
\usepackage{pbox}           
\usepackage{pdflscape}     
\usepackage{thmtools}      
\usepackage{afterpage}     
\usepackage{empheq}       
\usepackage{booktabs}      
\usepackage{url}               
\usepackage{bbm}             
\usepackage{physics}         
\usepackage{enumitem}     
\usepackage{mathrsfs}       
\usepackage{tikz}
\usetikzlibrary{patterns}
\usetikzlibrary{arrows.meta}

		\newtheorem{thm}{Theorem}[section]
		\newtheorem{lem}[thm]{Lemma}

			\newtheorem{dfn}[thm]{Definition}
			\newtheorem{rem}[thm]{Remark}

			\newtheorem*{thm*}{Theorem}
			\newtheorem*{lem*}{Lemma}
			\newtheorem*{prop*}{Proposition}
			\newtheorem*{cor*}{Corollary}
			\newtheorem*{dfn*}{Definition}
			\newtheorem*{rem*}{Remark}
			\newtheorem*{exa*}{Example}
			\newtheorem*{notation*}{Notation}
			\newtheorem*{terminology*}{Terminology}

\DeclareMathOperator{\supp}{supp}

\newcommand{\bigO}{\mathcal{O}}

\newcommand{\naturals}{\mathbb{N}}
\newcommand{\reals}{\mathbb{R}}

\newcommand{\dee}{\, \mathrm{d}}

\newcommand{\ee}{\mathrm{e}}
\newcommand{\restr}[2]{{\left.\kern-\nulldelimiterspace #1 \right|_{#2}}}

\renewcommand{\leq}{\leqslant}
\renewcommand{\geq}{\geqslant}
\renewcommand{\epsilon}{\varepsilon}
  

\usepackage{rotating}

\newcommand{\NN}[1]{\mathcal{N}\!\mathcal{N}_{n, m, #1}^\rho}

\newcommand{\Inm}{\mathcal{I}_{n, m, n+m+1}^\rho}
\newcommand{\NNn}{\mathcal{N}_{n}^\rho}
\newcommand{\chebnorm}[1]{\norm{#1}_\infty}

\newcommand{\cell}[1]{\multicolumn{1}{|c|}{\pbox[t]{50cm}{\vspace{2pt}#1\vspace{1pt}}}}

\def\slayer(#1, #2, #3, #4, #5, #6, #7){\cline{1-1}\cline{3-3}\cline{5-5}\cline{7-7}\cline{9-9}\cline{11-11}\cline{13-13}\cell{#1} &\,& \cell{#2} & \raisebox{-1.25em}{$\boldsymbol{\cdots}$} & \cell{#3} &\,& \cell{#4} &\,& \cell{#5} &\,& \cell{#6} & \raisebox{-1.25em}{$\boldsymbol{\cdots}$} &\cell{#7}\\\cline{1-1}\cline{3-3}\cline{5-5}\cline{7-7}\cline{9-9}\cline{11-11}\cline{13-13}}
\def\layer(#1, #2, #3, #4, #5, #6, #7){\slayer(#1, #2, #3, #4, #5, #6, #7)\\}
\def\llayer(#1, #2, #3){\cline{9-9}\cline{11-11}\cline{13-13}&\,&&&&\,&&\,& \cell{#1} &\,& \cell{#2} & \raisebox{-0.7em}{$\boldsymbol{\cdots}$} &\cell{#3}\\\cline{9-9}\cline{11-11}\cline{13-13}\\}
\def\flayer(#1, #2, #3){\cline{1-1}\cline{3-3}\cline{5-5}\cell{#1} &\,& \cell{#2} & \raisebox{-0.7em}{$\boldsymbol{\cdots}$} & \cell{#3}&&&&&&.\\\cline{1-1}\cline{3-3}\cline{5-5}\\}

\newcommand{\y}[3]{$\gamma_{#1, #2} =$\\$ \iota(\gamma_{#1, #3})$}

\newcommand{\sss}[3]{$\tau_{#2} = \sigma_{#1, #2}(\gamma_{1, #3}{,} \ldots{,} \gamma_{n, #3})$}

\newcommand{\uu}[3]{$\tau_{\beta_1 + #2} =$\\$ \sigma_{#1, #2}(\gamma_{1, #3}{,} \ldots{,} \gamma_{n, #3})$}
\newcommand{\uuu}[3]{$\tau_{\beta_1 + #2} = \sigma_{#1, #2}(\gamma_{1, #3}{,} \ldots{,} \gamma_{n, #3})$}

\newcommand{\vv}[4]{$\tau_{#4} =$\\$ \sigma_{#1, #2}(\gamma_{1, #3}{,} \ldots{,} \gamma_{n, #3})$}

\newcommand{\rotatedcdots}{\rotatebox[origin=c]{90}{$\boldsymbol{\cdots}$}}  

\def\simplelayer(#1, #2, #3, #4, #5){\cline{1-1}\cline{3-3}\cline{5-5}\cline{7-7}\cline{9-9}\cell{#1\vphantom{#2}\vphantom{#3}\vphantom{#4}\vphantom{#5}}&\,&\cell{#2\vphantom{#1}\vphantom{#3}\vphantom{#4}\vphantom{#5}}&\raisebox{-1.25em}{$\boldsymbol{\cdots}$} & \cell{#3\vphantom{#2}\vphantom{#1}\vphantom{#4}\vphantom{#5}} &\,& \cell{#4\vphantom{#2}\vphantom{#3}\vphantom{#1}\vphantom{#5}} &\,& \cell{#5\vphantom{#2}\vphantom{#3}\vphantom{#4}\vphantom{#1}}\\\cline{1-1}\cline{3-3}\cline{5-5}\cline{7-7}\cline{9-9}}
\def\simpleflayer(#1, #2, #3){\cline{1-1}\cline{3-3}\cline{5-5}\cell{#1} &\,& \cell{#2} & \raisebox{-0.7em}{$\boldsymbol{\cdots}$} & \cell{#3}\\\cline{1-1}\cline{3-3}\cline{5-5}\\}
\def\simplellayer(#1){\cline{9-9}&&&&&&&&\cell{#1}\\\cline{9-9}}

\title{Universal Approximation with Deep Narrow Networks}
\coltauthor{%
\Name{Patrick Kidger} \Email{kidger@maths.ox.ac.uk}\\
\Name{Terry Lyons} \Email{tlyons@maths.ox.ac.uk}\\
\addr Mathematical Institute, University of Oxford%
}

\begin{document}
\maketitle

\begin{abstract}%
The classical Universal Approximation Theorem holds for neural networks of arbitrary width and bounded depth. Here we consider the natural `dual' scenario for networks of bounded width and arbitrary depth. Precisely, let $n$ be the number of inputs neurons, $m$ be the number of output neurons, and let $\rho$ be any nonaffine continuous function, with a continuous nonzero derivative at some point. Then we show that the class of neural networks of arbitrary depth, width $n + m + 2$, and activation function $\rho$, is dense in $C(K; \reals^m)$ for $K \subseteq \reals^n$ with $K$ compact. This covers every activation function possible to use in practice, and also includes polynomial activation functions, which is unlike the classical version of the theorem, and provides a qualitative difference between deep narrow networks and shallow wide networks. 
We then consider several extensions of this result. In particular we consider nowhere differentiable activation functions, density in noncompact domains with respect to the $L^p$-norm, and how the width may be reduced to just $n + m + 1$ for `most' activation functions.
\end{abstract}


\section{Introduction}
Recall the classical Universal Approximation Theorem \citep{cybenko, hornik, pinkus}.

\begin{restatable}{thm}{Cybenko}\label{Cybenko} Let $\rho \colon \reals \to \reals$ be any continuous function. Let $\NNn$ represent the class of feedforward neural networks with activation function $\rho$, with $n$ neurons in the input layer, one neuron in the output layer, and one hidden layer with an arbitrary number of neurons. Let $K \subseteq \reals^n$ be compact. Then $\NNn$ is dense in $C(K)$ if and only if $\rho$ is nonpolynomial.
\end{restatable}

Extending this result to any bounded number of hidden layers is easy, by simply requiring that the `extra' hidden layers approximate the identity function. Thus the classical theorem addresses the case of arbitrary width and bounded depth.

This motivates a natural `dual' scenario, in which the class of neural networks is of bounded width and arbitrary depth. We refer to networks of this type as \emph{deep, narrow} networks. Natural questions are then what activation functions may be admitted, in what topologies density may be established, and how narrow the network may be made.

%
%

Notable existing work on this problem has been performed by \citet{LuEtAl} and \citet{Hanin2017}, however both of these studies only consider the ReLU activation function. In particular they rely on its explicit form and friendly algebraic properties.

The primary aim of this article is to addess this limitation, by considering essentially arbitrary activation functions. In particular we will find that polynomial activation functions are valid choices, meaning that deep and narrow networks behave distinctly differently to shallow and wide networks. We also provide some results on the choice of topology and the width of the network.

The rest of the paper is laid out as follows. Section 2 discusses existing work. Section 3 provides a summary of our results; these are then presented in detail in Section 4. Section 5 is the conclusion. A few proofs are deferred to the appendices.
 
\section{Existing work}
Some positive results have been established showing density of particular deep narrow networks.

\citet{Hanin2017} have shown that deep narrow networks with the ReLU activation function are dense in $C(K; \reals^m)$ for $K \subseteq \reals^n$ compact, and require only width $n + m$. \citet{LuEtAl} have shown that deep narrow networks with the ReLU activation function are dense in $L^1(\reals^n)$, with width $n + 4$. \citet{resnet} have shown that a particular description of residual networks, with the ReLU activation function, are dense in $L^1(\reals^n)$.

We are not aware of any previously obtained positive results for activation functions other than the ReLU, or for the general case of $L^p(\reals^n; \reals^m)$ for $p \in [1, \infty)$ and $m \in \naturals$.

Moving on, some negative results have been established, about insufficiently wide networks.

Consider the case of a network with $n$ input neurons and a single output neuron. For certain activation functions, \citet{johnson} shows that width $n$ is insufficient to give density in $C(K)$. For the ReLU activation function, \citet{LuEtAl} show that width $n$ is insufficient to give density in $L^1(\reals^n)$, and that width $n-1$ is insufficient in $L^1([-1,1]^n)$, whilst \citet{Hanin2017} show that width $n$ is insufficient to give density in $C(K)$.



Everything discussed so far is in the most general case of approximating functions on Euclidean space. There has also been some related work for classification tasks \citep{class1, class2, class3, nguyen}. There has also been some related work in the special case of certain finite domains; \citet{other1, deepbelief} consider distributions on $\{0,1\}^n$. 
\citet{guido} consider distributions on $\{0, 1, \ldots, q-1\}^n$.

\section{Summary of Results}
\begin{dfn}
Let $\rho \colon \reals \to \reals$ and $n, m, k \in \naturals$. Then let $\NN{k}$ represent the class of functions $\reals^n \to \reals^m$ described by feedforward neural networks with $n$ neurons in the input layer, $m$ neurons in the output layer, and an arbitrary number of hidden layers, each with $k$ neurons with activation function $\rho$. Every 
neuron in the output layer has the identity activation function.
\end{dfn}

Our main result is the following theorem.
\begin{restatable}{thm}{megathm}\label{megathm}
Let $\rho \colon \reals \to \reals$ be any nonaffine continuous function which is continuously differentiable at at least one point, with nonzero derivative at that point. Let $K \subseteq \reals^n$ be compact. Then $\NN{n + m + 2}$ is dense in $C(K; \reals^m)$ with respect to the uniform norm.
\end{restatable}
The key novelty here is the ability to handle essentially arbitrary activation functions, and in particular polynomials, which is a qualitative difference compared to shallow networks. In particular we have not relied on the explicit form of the ReLU, or on its favourable algebraic properties.

The technical condition is very weak; in particular it is satisfied by every piecewise-$C^1$ function not identically zero. Thus any activation function that one might practically imagine using on a computer must satisfy this property.

Theorem \ref{megathm} is proved by handling particular classes of activation functions as special cases.
\begin{restatable*}{prop}{mainthm}\label{mainthm}
Let $\rho \colon \reals \to \reals$ be any continuous nonpolynomial function which is continuously differentiable at at least one point, with nonzero derivative at that point. Let $K \subseteq \reals^n$ be compact. Then $\NN{n + m + 1}$ is dense in $C(K; \reals^m)$ with respect to the uniform norm.
\end{restatable*}
\begin{restatable*}{prop}{polythmtwo}\label{polythmtwo}
Let $\rho \colon \reals \to \reals$ be any nonaffine polynomial. Let $K \subseteq \reals^n$ be compact. Then $\NN{n + m + 2}$ is dense in $C(K; \reals^m)$ with respect to the uniform norm.
\end{restatable*}
It is clear that Propositions \ref{mainthm} and \ref{polythmtwo} together imply Theorem \ref{megathm}. Note the slight difference in their required widths. Furthermore we will see that their manner of proofs are rather different.

As a related result, some of the techniques used may also be shown to extend to certain unfriendly-looking activation functions, that do not satisfy the technical condition of Theorem \ref{megathm}.
\begin{restatable*}{prop}{nondiffthm}\label{nondiffthm}
Let $w \colon \reals \to \reals$ be any bounded continuous nowhere differentiable function. Let $\rho(x) = \sin(x) + w(x)\ee^{-x}$, which will also be nowhere differentiable. Let $K \subseteq \reals^n$ be compact. Then $\NN{n + m + 1}$ is dense in $C(K; \reals^m)$ with respect to the uniform norm.
\end{restatable*}

Moving on to a different related result, we consider the case of a noncompact domain.
\begin{restatable*}{thm}{relucor}\label{relucor}
Let $\rho$ be the ReLU. Let $p \in [1, \infty)$. Then $\NN{n + m + 1}$ is dense in $L^p(\reals^n; \reals^m)$ with respect to the usual $L^p$ norm.
\end{restatable*}
Whilst only about the ReLU, the novelty of this result is how it generalises \citep[Theorem 1]{LuEtAl} in multiple ways: to a narrower width, multiple outputs, and $L^p$ instead of just $L^1$.

As a final related result, we observe that the smaller width of $n + m + 1$ also suffices for a large class of polynomials. Together with Proposition \ref{mainthm}, this means that the smaller width of $n + m + 1$ suffices for `most' activation functions.
\begin{restatable*}{prop}{polythm}\label{polythm}
Let $\rho \colon \reals \to \reals$ be any polynomial for which there exists a point $\alpha \in \reals$ such that $\rho'(\alpha) = 0$ and $\rho''(\alpha) \neq 0$. Let $K \subseteq \reals^n$ be compact. Then $\NN{n + m + 1}$ is dense in $C(K; \reals^m)$ with respect to the uniform norm.
\end{restatable*}

\begin{rem}
Every proof in this article is constructive, and can in principle be traced so as to determine how depth changes with approximation error. We have elected not to present this, as depth-efficient versions of our constructions quickly become unclear to present. Furthermore this would require tracing the (constructive) proofs of the Stone--Weierstrass Theorem and the classical Universal Approximation Theorem, to which we appeal.

\end{rem}


\section{Universal approximation}
\subsection{Preliminaries}
A neuron is usually defined as an activation function composed with an affine function. For ease, we shall extend the definition of a neuron to allow it to represent a function of the form $\psi \circ \rho \circ \phi$, where $\psi$ and $\phi$ are affine functions, and $\rho$ is the activation function. This does not increase the representational power of the network, as the new affine functions may be absorbed into the affine parts of the next layer, but it will make the neural representation of many functions easier to present. We refer to these as \emph{enhanced neurons}. It is similarly allowable to take affine combinations of multiple enhanced neurons; we will use this fact as well.

One of the key ideas behind our constructions is that most reasonable activation functions can be taken to approximate the identity function. Indeed, this is essentially the notion that differentiability captures: that a function is locally affine. This makes it possible to treat neurons as `registers', in which information may be stored and preserved through the layers. 
Thus our constructions have strong overtones of space-limited algorithm design in traditional computer science settings; in our proofs we will often think of `storing' a value in a particular register.

\begin{lem}\label{identity1}
Let $\rho \colon \reals \to \reals$ be any continuous function which is continuously differentiable at at least one point, with nonzero derivative at that point. Let $L \subseteq \reals$ be compact. Then a single enhanced neuron with activation function $\rho$ may uniformly approximate the identity function $\iota \colon \reals \to \reals$ on $L$, with arbitrarily small error.
\end{lem}
\begin{proof}
By assumption, as $\rho$ is \emph{continuously} differentiable, there exists $[a, b] \subseteq \reals$ with $a \neq b$, on some neighbourhood of which $\rho$ is differentiable, and $\alpha \in (a, b)$ at which $\rho'$ is continuous, and for which $\rho'(\alpha)$ is nonzero.

For $h \in \reals \setminus \{0\}$, let $\phi_h(x) = hx + \alpha$, and let
\begin{equation*}
\psi_h(x) =\frac{x -\rho(\alpha)}{h \rho'(\alpha)}.
\end{equation*}
Then
$\iota_h = \psi_h \circ \rho \circ \phi_h$
is of the form that an enhanced neuron can represent. Then for all $u \in [a, b]$, by the Mean Value Theorem there exists $\xi_u$ between $u$ and $\alpha$ such that
\begin{align*}
\rho(u) &= \rho(\alpha) + (u - \alpha)\rho'(\xi_u),
\end{align*}
and hence
\begin{align*}
\iota_h(x) &= (\psi_h \circ \rho \circ \phi_h)(x) \\
&= \psi_h\left(\rho(\alpha) + hx\rho'(\xi_{hx + \alpha})\right)\\
 &= \frac{x \rho'(\xi_{hx + \alpha})}{\rho'(\alpha)}
\end{align*}
for $h$ sufficiently small that $\phi_h(L) \subseteq [a, b]$.

Now let $\rho'$ have modulus of continuity $\omega$ on $[a, b]$. Let $\iota \colon \reals \to \reals$ represent the identity function. Then for all $x \in L$,
\begin{align*}
\abs{\iota_h(x) - \iota(x)} &= \abs{x}\abs{\frac{\rho'(\xi_{hx + \alpha}) - \rho'(\alpha)}{\rho'(\alpha)}}\\
&\leq \frac{\abs{x}}{\abs{\rho'(\alpha)}} \omega(hx),
\end{align*}
and so $\iota_h \to \iota$ uniformly over $L$.
\end{proof}
\begin{notation*}
Throughout the rest of this paper $\iota_h$ will be used to denote such an approximation to the identity function, where $\iota_h \to \iota$ uniformly as $h \to 0$.
\end{notation*}
An enhanced neuron may be described as performing (for example) the computation $x \mapsto \iota_h(4 x + 3)$. This is possible as the affine transformation $x \mapsto 4x + 3$ and the affine transformation $\phi_h$ (from the description of $\iota_h$) may be combined together into a single affine transformation.

Now that we can approximate the identity funciton, another important ingredient will be a model that actually uses identity functions for us to approximate! As such we now consider the `Register Model', which represents a simplification of a neural network.
\begin{restatable}[Register Model]{prop}{registermodel}\label{register-model}
Let $\rho \colon \reals \to \reals$ be any continuous nonpolynomial function. Let $\Inm$ represent the class of neural networks with $n$ neurons in the input layer, $m$ neurons in the output layer, and an arbitrary number of hidden layers, each with $n + m$ neurons with the identity activation function, and one neuron with activation function $\rho$. Let $K \subseteq \reals^n$ be compact. Then $\Inm$ is dense in $C(K; \reals^m)$.
\end{restatable}
The Register Model is somewhat similar to the constructions used in \citep{LuEtAl, Hanin2017}, although their constructions are specific to the ReLU. As such we defer the proof to Appendix \ref{register-model-appendix}.

Next we consider another different simplification, which we refer to as the `Square Model'. Its proof is a little involved, so for clarity we present it in a new subsection.
\subsection{Square Model}

\begin{lem}\label{mult}
One layer of two enhanced neurons, with square activation function, may exactly represent the multiplication function $(x, y) \mapsto xy$ on $\reals^2$.
\end{lem}
\begin{proof}
Let the first neuron compute $\eta = (x + y)^2$. Let the second neuron compute $\zeta = (x - y)^2$. Then $xy = (\eta - \zeta) / 4$. (This final affine transformation is allowed between enhanced neurons.)
\end{proof}

\begin{lem}\label{fused}
Fix $L \subseteq \reals^2$ compact. Three layers of two enhanced neurons each, with square activation function, may uniformly approximate $(x, y) \mapsto (x^2, y(x + 1))$ arbitrarily well on $L$.
\end{lem}
\begin{proof}
Let $h, s \in \reals \setminus \{0\}$. Let $\eta_1, \eta_2, \eta_3$ represent the first neuron in each layer; let $\zeta_1, \zeta_2, \zeta_3$ represent the second neuron in each layer. Let $\iota_h$ represent an approximation to the identity in the manner of Lemma \ref{identity1}. Using `$\approx$' as an informal notation to represent `equal to up to approximation of the identity', 
assign values to $\eta_1, \eta_2, \eta_3$ and $\zeta_1, \zeta_2, \zeta_3$ as follows:
\begin{align*}
\eta_1 &= \iota_h(x) &\zeta_1 &= (x + sy + 1)^2\\
 &\approx x,&&= x^2 + 2sxy + s^2 y^2 + 2x + 2sy + 1,\\
 \eta_2 &= (\eta_1)^2 & \zeta_2 &= \iota_h(\zeta_1 - 2 \eta_1 - 1) \\
 &\approx x^2,&& \approx x^2 + 2sxy + s^2 y^2 + 2sy,\\
 \eta_3 &= \iota_h(\eta_2) & \zeta_3 &=\iota_h((\zeta_2 - \eta_2)/2s)\\
 &\approx x^2, && \approx xy + y + s y^2/2.
\end{align*}

And so $\eta_3$ may be taken arbitrarily close to $x^2$ and $\zeta_3$ may be taken arbitrarily close to $y(x+1)$, with respect to $\chebnorm{\,\cdot\,}$ on $L$, by first taking $s$ arbitrarily small, and then taking $h$ arbitrarily small.
\end{proof}

\begin{lem}\label{reciprocal}
Fix $L \subseteq (0, 2)$ compact. Then multiple layers of two enhanced neurons each, with square activation function, may uniformly approximate $x \mapsto 1/x$ arbitrarily well on $L$.
\end{lem}
\begin{proof}
First note that
\begin{equation*}
\prod_{i = 0}^n (1 + x^{2^i}) \to \frac{1}{1 - x}
\end{equation*}
as $n \to \infty$, uniformly over compact subsets of $(-1, 1)$. 
Thus,
\begin{equation*}
(2 - x) \prod_{i = 1}^n (1 + (1 - x)^{2^i}) = \prod_{i = 0}^n (1 + (1 - x)^{2^i}) \to \frac{1}{x}
\end{equation*}
uniformly over $L$.

This has the following neural approximation: let $\eta_1 = (1-x)^2$ and $\zeta_1 = \iota_h(2 - x)$ be the neurons in the first layer, where $\iota_h$ is some approximation of the identity as in Lemma \ref{identity1}. Let $\kappa_h$ represent an approximation to $(x, y) \mapsto (x^2, y(x + 1))$ in the manner of Lemma \ref{fused}, with error made arbitrarily small as $h \to 0$. Now for $i \in \{1, 4, 7, 10, \ldots, 3n - 2\}$, recursively define $(\eta_{i + 3}, \zeta_{i + 3}) = \kappa_h(\eta_{i}, \zeta_{i})$, where we increase the index by three to represent the fact that three layers are used to perform this operation. So up to approximation, $\eta_{i+3} \approx (\eta_{i})^2$, and $\zeta_{i + 3} \approx \zeta_i (\eta_i + 1)$.

So $\zeta_{3n + 1} \to (2 - x) \prod_{i = 1}^n (1 + (1 - x)^{2^i})$ uniformly over $L$ as $h \to 0$. Thus the result is obtained by taking first $n$ large enough and then $h$ small enough.
\end{proof}


\begin{restatable}[Square Model]{prop}{squareprop}\label{squareprop}
Let $\rho(x) = x^2$. Let $K \subseteq \reals^n$ be compact. Then $\NN{n + m + 1}$ is dense in $C(K; \reals^m)$.
\end{restatable}
\begin{proof}
Fix $f = (f_1, \ldots, f_m) \in C(K; \reals^m)$. Fix $\epsilon > 0$. By precomposing with an affine function, which may be absorbed into the first layer of the network, assume without loss of generality that
\begin{equation}\label{eq:kassume}
K \subseteq (1, 2)^n.
\end{equation}
By the Stone--Weierstrass Theorem there exist polynomials $g_1, \ldots, g_m$ in the variables $x_1, \ldots x_n$ approximating $f_1, \ldots, f_m$ to within $\epsilon/3$ with respect to $\chebnorm{\,\cdot\,}$.

We will construct a network in $\NN{n + m + 1}$ approximating $g_1, \ldots, g_n$. There will be a total of $n + m + 1$ neurons in each hidden layer. 
In each hidden layer, for each $i \in \{1, \ldots, n\}$, we associate an input $x_i$ with a neuron, which we shall refer to as the $x_i$-in-register neuron. Similarly for each $i \in \{1, \ldots, m\}$, we will associate an output $g_i$ and a neuron, which we shall refer to as the $g_i$-out-register neuron. The final neuron in each layer will be referred to as the computation neuron.

Our proof will progress by successively adding layers to the network. 

We begin by constructing approximations to $g_2, \ldots, g_m$. (Constructing the approximation to the final $g_1$ will be more challenging.) These next few paragraphs may be skipped if $m = 1$.

In each hidden layer, for $i \in \{1, \ldots, n\}$, have the $x_i$-in-register neuron apply the approximate identity function in the manner of Lemma \ref{identity1} to the $x_i$-in-register neuron of the previous layer, or to the input $x_i$ if it is the first hidden layer. In this way the in-register neurons will preserve the inputs to the network, so that the values of $x_i$ are accessible by the other neurons in every layer.


(At least, up to an arbitrarily good approximation of the identity. 
For the sake of sanity of notation, we shall suppress this detail in our notation, and refer to our neurons in later layers as having e.g. `$x_1$' as an input to them.)

Now write $g_2 = \sum_{j = 1}^N \delta_j$, where each $\delta_j$ is a monomial. Using just the computation neuron and the $g_1$-out-register neuron in multiple consecutive layers, perform successive multiplications in the manner of Lemma \ref{mult} 
to compute the value of $\delta_1$. For example, if $\delta_1 = x_1^2 x_2 x_3$, then a suitable chain of multiplications is $x_1 (x_1 (x_2 x_3))$. In each layer, each $x_i$ is available as an input because it is stored in the in-register neurons, and the intermediate partial products are available as they have just been computed in the preceding layer. 
The value for $\delta_1$ is then stored in the $g_2$-out-register neuron 
and kept through the subsequent layers via approximate identity functions.

The previous paragraph is somewhat wordy, so it is worth pausing to make clear what the appropriate mental model is for validating that this construction is correct. It is simply that at each layer, we have at most $n + m + 1$ values available from the preceding layer, and we must now allocate a budget of at most $n + m + 1$ enhanced neurons, describing how to obtain the at most $n + m + 1$ outputs from the layer. Values may be copied, moved and added between neurons using the affine part of a layer, and preserved between layers using $\iota_h$. We have not yet needed our full budget of neurons, so far.

This process is then repeated for $\delta_2$, again using just the computation neuron and the $g_1$-out-register neuron. The result is then added on to the $g_2$-out-register neuron, via the affine part of the operation of this neuron. Repeat for all $j$ until all of the $\delta_j$ have been computed and added on, so that the $g_2$-out-register neuron stores an approximation to $g_2$.

This is only an approximation in that we have used approximations to the identity function; other than that it is exact. As such, by taking sufficiently good approximations of the identity function, this will be a uniform approximation to $g_2$ over $K$. For the remaining layers of the network, have the $g_2$-out-register neurons just preserve this value with approximate identity functions.

Now repeat this whole process for $g_i$ and the $g_i$-out-register neurons, for $i \in \{3, \ldots, m\}$. 
Let the computed values be denoted $\widehat{g}_2, \ldots, \widehat{g}_m$. (With the `hat' notation because of the fact that these are not the values $g_2, \ldots, g_m$, due to the approximate identity functions in between.)

The difficult bit is computing an approximation to $g_1$, as it must be done without the `extra' $g_1$-in-register neuron. Going forward we now only have $n + 2$ neurons available in each layer: the $n$ in-register neurons (which have so far been storing the inputs $x_1, \ldots x_n$), the computation neuron, and the $g_1$-out-register neuron.

Written in terms of monomials, let $g_1 = \sum_{j = 1}^M \gamma_{j}$. Then $g_1$ may be written as
\begin{equation*}
g_1 = \gamma_{1} \left(1 + \frac{\gamma_{2}}{\gamma_{1}}\left(1 + \frac{\gamma_{3}}{\gamma_{2}}\left(\cdots \left(1 + \frac{\gamma_{M-1}}{\gamma_{M-2}}\left(1 + \frac{\gamma_{M}}{\gamma_{M - 1}}\right)\right) \cdots \right)\right)\right).
\end{equation*}
Note that this description is defined over $K$, as $K$ is bounded away from the origin by equation \eqref{eq:kassume}.

Now write $\gamma_{j} = \prod_{k = 1}^{n} x_k^{\theta_{j, k}}$, for $\theta_{j, k} \in \naturals_{0}$. Substituting this in,
\begin{align*}
g_1 = \left[\prod_{k = 1}^{n} x_k^{\theta_{1, k}}\right] \Bigg(1 + \frac{\prod_{k = 1}^{n} x_k^{\theta_{2, k}}}{\prod_{k = 1}^{n} x_k^{\theta_{1, k}}}&\Bigg(1 + \frac{\prod_{k = 1}^{n} x_k^{\theta_{3, k}}}{\prod_{k = 1}^{n} x_k^{\theta_{2, k}}}\Bigg(\cdots \\
&\left(1 + \frac{\prod_{k = 1}^{n} x_k^{\theta_{M-1, k}}}{\prod_{k = 1}^{n} x_k^{\theta_{M - 2, k}}}\left(1 + \frac{\prod_{k = 1}^{n} x_k^{\theta_{M, k}}}{\prod_{k = 1}^{n} x_k^{\theta_{M-1, k}}}\right)\right) \cdots \Bigg)\Bigg)\Bigg).
\end{align*}

Now let $\sup K$ be defined by
\begin{equation*}
\sup K = \sup \{ x_i \,\vert\, (x_1, \ldots, x_n) \in K\},
\end{equation*}
so that $1 < \sup K < 2$. Let $r$ be an approximation to $x \mapsto 1/x$ in the manner of Lemma \ref{reciprocal}, with the $L$ of that proposition given by
\begin{equation}\label{eq:lchoice}
L = [(\sup K)^{-1} - \alpha, \sup K + \alpha] \subseteq (0, 2),
\end{equation}
where $\alpha > 0$ is taken small enough that the inclusion holds.

Let $r^a$ denote $r$ composed $a$ times. By taking $r$ to be a suitably good approximation, we may ensure that $\widetilde{g}_1$ defined by
\begin{align}
\widetilde{g}_1 = &\left[\prod_{k = 1}^{n} r^{2M - 2}(x_k)^{\theta_{1, k}}\right] \Bigg(1 + \left[\prod_{k = 1}^{n} r^{2M - 3}(x_k)^{\theta_{1, k}}\right]\left[\prod_{k = 1}^{n}r^{2M - 4}(x_k)^{\theta_{2, k}}\right]\nonumber\\
&\hspace{9.35em}\Bigg(1 + \left[\prod_{k = 1}^{n} r^{2M - 5}(x_k)^{\theta_{2, k}}\right]\left[\prod_{k = 1}^{n}r^{2M - 6}(x_k)^{\theta_{3, k}}\right]\\&\hspace{10em}\Bigg(\cdots\nonumber\\
&\hspace{10.65em}\Bigg(1 + \left[\prod_{k = 1}^{n}r^3(x_k)^{\theta_{M-2, k}}\right]\left[\prod_{k = 1}^{n} r^2(x_k)^{\theta_{M-1, k}}\right]\nonumber\\
&\hspace{11.3em}\left(1 + \left[\prod_{k = 1}^{n} r(x_k)^{\theta_{M-1, k}}\right]\left[\prod_{k = 1}^{n} x_k^{\theta_{M, k}}\right]\right)\Bigg)\nonumber\\
&\hspace{13em}\cdots \Bigg)\Bigg)\Bigg)\label{eq:tildeg}
\end{align}
is an approximation to $g_1$ in $K$, to within $\epsilon/3$, with respect to $\chebnorm{\,\cdot\,}$. This is possible by equations (\ref{eq:kassume}) and (\ref{eq:lchoice}); in particular the approximation should be sufficiently precise that
\begin{equation*}
r^{2M - 2}([(\sup K)^{-1}, \sup K]) \subseteq L,
\end{equation*}
which is possible due to the margin $\alpha > 0$. Note how $r^2$, and thus $r^4$, $r^6$, \ldots, $r^{2M - 2}$, are approximately the identity function on $L$.

This description of $\widetilde{g}_1$ is now amenable to representation with a neural network. The key fact about this description of $\widetilde{g}_1$ is that, working from the most nested set of brackets outwards, the value of $\widetilde{g}_1$ may be computed by performing a single chain of multiplications and additions, along with occasionally taking the reciprocal of all of the input values. Thus unlike our earlier computations involving the monomial $\delta_j$, we do not need to preserve an extra partially-computed piece of information between layers.

So let the computation neuron and the $g_1$-out-register neuron perform the multiplications, layer-by-layer, to compute $\prod_{k = 1}^{n} x_k^{\theta_{M, k}}$, in the manner of Lemma \ref{mult}. Store this value in the $g_1$-out-register neuron.

Now use the computation neuron and the $x_1$-in-register neuron, across multiple layers, to compute $r(x_1)$, in the manner of Lemma \ref{reciprocal}. Eventually the $x_1$-in-register neuron will be storing $r(x_1) \approx 1/x_1$. Repeat for the other in-register neurons, so that they are collectively storing $r(x_1), \ldots, r(x_n)$.

Now the computation neuron and the $g_1$-out-register neuron may start multipying $r(x_1), \ldots, r(x_n)$ on to $\prod_{k = 1}^{n} x_k^{\theta_{M, k}}$ (which is the value presently stored in the $g_1$-out-register neuron) the appropriate number to times to compute $\left[\prod_{k = 1}^{n} r(x_k)^{\theta_{M-1, k}}\right]\left[\prod_{k = 1}^{n} x_k^{\theta_{M, k}}\right]$, by Lemma \ref{mult}. Store this value in the out-register neuron. Then add one (using the affine part of a layer). The out-register neuron has now computed the expression in the innermost bracket in equation \eqref{eq:tildeg}.

The general pattern is now clear: apply $r$ to all of the in-register neurons again to compute $r^2(x_i)$, multiply them on to the value in the out-register neuron, and so on. Eventually the $g_1$-out-register neuron will have computed an approximation to $\widetilde{g}_1$. It will have computed some approximation $\widehat{g}_1$ to this value, because of the identity approximations involved.

Thus the out-register neurons have computed $\widehat{g}_1, \ldots, \widehat{g}_m$. Now simply have the output layer copy the values from the out-register neurons.

Uniform continuity preserves uniform convergence, compactness is preserved by continuous functions, and a composition of two uniformly convergent sequences of functions with uniformly continuous limits is again uniformly convergent. So by taking all of the (many) identity approximations throughout the network to be suitably precise, then $\widetilde{g}_1$ and $\widehat{g}_1$ may be taken within $\epsilon/3$ of each other, and the values of $\widehat{g}_2, \ldots, \widehat{g}_m$ and $g_2, \ldots, g_m$ may be taken within $2 \epsilon / 3$ of each other, in each case with respect to $\chebnorm{\,\cdot\,}$ on $K$.


Thus $(\widehat{g}_1, \ldots, \widehat{g}_m)$ approximates $f$ with total error no more than $\epsilon$, and the proof is complete.
\end{proof}

\begin{rem}\label{divrem}
Lemma \ref{reciprocal} is key to the proof of Proposition \ref{squareprop}. It was fortunate that the reciprocal function may be approximated by a network of width two - note that even if Proposition \ref{squareprop} were already known, it would have required a network of width three. It remains unclear whether an arbitrary-depth network of width two, with square activation function, is dense in $C(K)$.
\end{rem}

\begin{rem}\label{rema}
Note that allowing a single extra neuron in each layer would remove the need for the trick with the reciprocal, as it would allow $g_1$ to be computed in the same way as $g_2, \ldots, g_m$. Doing so would dramatically reduce the depth of the network. We are thus paying a heavy price in depth in order to reduce the width by a single neuron.
\end{rem}



\subsection{Key results}

\mainthm
\begin{proof}
Let $f \in C(K; \reals^m)$ and $\epsilon > 0$. Set up a neural network as in the Register Model (Proposition \ref{register-model}), approximating $f$ to within $\epsilon / 2$. Every neuron requiring an identity activation function in the Register Model will instead approximate the identity with $\iota_h$, in the manner of Lemma \ref{identity1}.

Uniform continuity preserves uniform convergence, compactness is preserved by continuous functions, and a composition of two uniformly convergent sequences of functions with uniformly continuous limits is again uniformly convergent. 
Thus the new model can be taken within $\epsilon / 2$ of the Register Model, with respect to $\chebnorm{\,\cdot\,}$ in $K$, by taking $h$ sufficiently small.
\end{proof}

\begin{rem}\label{relucorrem}
This of course implies approximation in $L^p(K, \reals^m)$ for $p \in [1, \infty)$. However, when $\rho$ is the ReLU activation function, then Theorem \ref{relucor}, later, shows that in fact the result may be generalised to noncompact domains.
\end{rem}


Moving on, we consider polynomial activation functions. 
We now show that it is a consequence of Proposition \ref{squareprop} that any (polynomial) activation function which can approximate the square activation function, in a suitable manner, is also capable of universal approximation.

\polythmtwo
\begin{proof}
Fix $\alpha \in \reals$ such that $\rho''(\alpha) \neq 0$, which exists as $\rho$ is nonaffine. Now let $h \in (0, \infty)$. Define $\sigma_h \colon \reals \to \reals$ by
\begin{equation*}
\sigma_h(x) = \frac{\rho(\alpha + hx) - 2\rho(\alpha) + \rho(\alpha - hx)}{h^2 \rho''(\alpha)}.
\end{equation*}
Then Taylor expanding $\rho(\alpha + hx)$ and $\rho(\alpha - hx)$ around $\alpha$,
\begin{align*}
\sigma_h(x) =\null &\frac{\rho(\alpha) + hx \rho'(\alpha) + h^2 x^2 \rho''(\alpha)/2 +\bigO(h^3 x^3)}{h^2 \rho''(\alpha)} - \frac{2\rho(\alpha)}{h^2 \rho''(\alpha)} +\\
&\frac{\rho(\alpha) - hx \rho'(\alpha) + h^2 x^2 \rho''(\alpha)/2 + \bigO(h^3 x^3)}{h^2 \rho''(\alpha)}\\
=\null& x^2 + \bigO(hx^3).
\end{align*}
Observe that $\sigma_h$ needs precisely two operations of $\rho$ on (affine transformations of) $x$, and so may be computed by two enhanced neurons with activation function $\rho$. Thus the operation of a single enhanced neuron with square activation function may be approximated by two enhanced neurons with activation function $\rho$.

Let $N$ be any network as in the Square Model (Proposition \ref{squareprop})
. Let $\ell$ be any hidden layer of $N$; it contains $n + m + 1$ neurons. Let $\eta$ be a vector of the values of the neurons of the previous layer. Let $\phi_i$ be the affine part of the $i$th neuron of $\ell$, so that $\ell$ computes $\phi_1(\eta)^2, \ldots, \phi_{n + m + 1}(\eta)^2$. Then this may equivalently be calculated with $n + m + 1$ layers of $n + m + 1$ neurons each, with $n + m$ of the neurons in each of these new layers using the identity function, and one neuron using the square activation function. This is done by having the first of these new layers apply the $\phi_i$, and having the $i$th layer square the value of the $i$th neuron. See Figure \ref{layer-expand}.

\begin{figure}
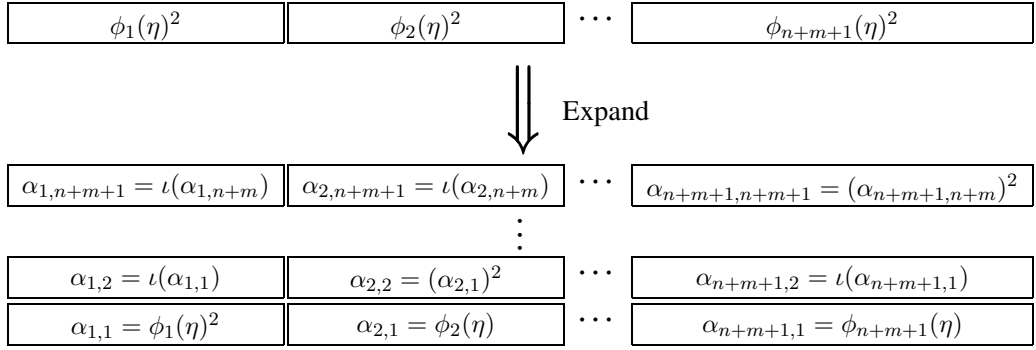
\centering
\resizebox{0.9\textwidth}{!}{
\begin{tabular}[t]{@{}c@{}c@{}ccc}
\simpleflayer({$\phi_1(\eta)^2$}, {$\phi_2(\eta)^2$}, {$\phi_{n + m + 1}(\eta)^2$})
\\\vspace{0pt}\\
\multicolumn{5}{c}{\vspace{3pt}\hphantom{\quad\large{Expand}} \scalebox{2}{\rotatebox[origin=c]{-90}{$\boldsymbol{\Longrightarrow}$}} \quad\large{Expand}}\\
\simpleflayer({$\alpha_{1, n + m + 1} = \iota(\alpha_{1, n + m})$}, {$\alpha_{2, n + m + 1} = \iota(\alpha_{2, n + m})$}, {$\alpha_{n + m + 1, n + m + 1} = (\alpha_{n + m + 1, n + m})^2$})
\multicolumn{5}{c}{\rotatedcdots\vspace{2pt}}\\
\simpleflayer({$\alpha_{1, 2} = \iota(\alpha_{1, 1})$}, {$\alpha_{2, 2} = (\alpha_{2, 1})^2$}, {$\alpha_{n + m + 1, 2} = \iota(\alpha_{n + m + 1, 1})$})
\simpleflayer({$\alpha_{1, 1} = \phi_1(\eta)^2$}, {$\alpha_{2, 1} = \phi_2(\eta)$}, {$\alpha_{n + m + 1, 1} = \phi_{n + m + 1}(\eta)$})
\end{tabular}}
\caption{A layer with square activation functions is equivalent to multiple layers with only a single square activation function in each layer. The other neurons use the identity activation function, denoted $\iota$.}\label{layer-expand}
\end{figure}

Apply this procedure to every layer of $N$; call the resulting network $\widetilde{N}$. It will compute exactly the same function as $N$, and will have $n + m + 1$ times as many layers, but will use only a single squaring operation in each layer.

Create a copy of $\widetilde{N}$, call it $\widetilde{N}_h$. Replace its identity activation functions with approximations in the manner of Lemma \ref{identity1}, using activation function $\rho$. Replace its square activation functions (one in each layer) by approximations in the manner described above with $\sigma_h$; this requires an extra neuron in each hidden layer, so that the network is now of width $n + m + 2$. Thus $\widetilde{N}_h$ uses the activation function $\rho$ throughout.

Uniform continuity preserves uniform convergence, compactness is preserved by continuous functions, and a composition of two uniformly convergent sequences of functions with uniformly continuous limits is again uniformly convergent. 
Thus the difference between $\widetilde{N}_h$ and $\widetilde{N}$, with respect to $\chebnorm{\,\cdot\,}$ on $K$, may be taken arbitrarily small by taking $h$ arbitrarily small.
\end{proof}

\begin{rem}\label{remb}
It is possible to construct shallower networks analogous to $\widetilde{N}$. This is because the proof of Proposition \ref{squareprop} actually uses many of the network's neurons to approximate the identity, so the identity activation functions of $\widetilde{N}$ may be used directly.
\end{rem}

\begin{rem}\label{rempoly}
That $\rho$ is polynomial is never really used in the proof of Proposition \ref{polythmtwo}. It is simply that a certain amount of differentiability is required, and all such nonpolynomial functions are already covered by Proposition \ref{mainthm}, as a nonzero second derivative at $\alpha$ implies a nonzero first derivative somewhere close to $\alpha$. Thus in principle this provides another possible construction by which certain networks may be shown to exhibit universal approximation.

\end{rem}

Given both Proposition \ref{mainthm} and Proposition \ref{polythmtwo}, then Theorem \ref{megathm} follows immediately.

\subsection{Extensions of these results to other special cases}
We begin by showing how the Register Model may also be used to handle the case of certain nowhere-differentiable activation functions. Our manner of proof may be extended to other nondifferentiable activation functions as well.



\begin{lem}\label{identity3}
Let $w \colon \reals \to \reals$ be any bounded continuous nowhere differentiable function. Let $\rho(x) = \sin(x) + w(x)\ee^{-x}$. Let $L \subseteq \reals$ be compact. Then a single enhanced neuron with activation function $\rho$ may uniformly approximate the identity function $\iota \colon \reals \to \reals$ on $L$, with arbitrarily small error.
\end{lem}
\begin{proof}
For $h \in \reals \setminus \{0\}$ and $A \in 2 \pi \naturals$, let $\phi_{h, A}(x) = hx  + A$, and let $\psi(x) = x/h$. Let
\begin{equation*}
\iota_{h, A} = \psi_h \circ \rho \circ \phi_{h, A},
\end{equation*}
which is of the form that an enhanced neuron can represent. Then jointly taking $h$ small enough and $A$ large enough it is clear that $\iota_{h, A}$ may be taken uniformly close to $\iota$ on $L$.
\end{proof}

\nondiffthm
\begin{proof}
As the proof of Proposition \ref{mainthm}, except substituting Lemma \ref{identity3} for Lemma \ref{identity1}.
\end{proof}

Next, we discuss how we may establish a result over a noncompact domain, by exploiting the nice properties of the ReLU.
\relucor
The full proof is deferred to Appendix \ref{relucorproof} due to space, but may be sketched as follows. Given some $f \in L^p(\reals^n; \reals^m)$, choose a compact set $K \subseteq \reals^n$ on which $f$ places most of its mass, and find a neural approximation to $f$ on $K$ in the manner of Proposition \ref{mainthm}. Once this is done, a cut-off function is applied outside the set, so that the network takes the value zero in $\reals^n \setminus K$. The interesting bit is finding a neural representation of such cut-off behaviour.

In particular the usual thing to do -- multiply by a cut-off function -- does not appear to have a suitable neural representation, as merely approximating the multiplication operation is not necessarily enough on an unbounded domain. Instead, the strategy is to 
take a maximum and a minimum with multiples of the cut-off function, which may be performed exactly. 

Moving on, our final result is that Proposition \ref{polythmtwo} may be improved upon, provided the activation function satisfies a certain condition.
\polythm
The proof is similar to that of Proposition \ref{polythmtwo}, so it is deferred to Appendix \ref{polythmproof}. Together with Proposition \ref{mainthm}, this means that `most' activation functions require a width of only $n + m + 1$.

\section{Conclusion}
There is a large literature on theoretical properties of neural networks, but much of it deals only with the ReLU.\footnote{See for example \citet{Hanin2017, philipp1, philipp2, daubechies, understandingrelu, relu1, relu2, relu3, relu4}.} However how to select an activation function remains a poorly understood topic, and many other options have been proposed: leaky ReLU, PReLU, RRelu,  ELU, SELU and other more exotic activation functions as well.\footnote{See \citet{leakyrelu, prelu, rrelu, elu, selu, relusix, activation3, activation2, activation1, padeactivation} respectively.}

Our central contribution is to provide results for universal approximation using general activation functions (Theorem \ref{megathm} and Propositions \ref{mainthm} and \ref{polythmtwo}). In contrast to previous work, these results do not rely on the nice properties of the ReLU, and in particular do not rely on its explicit description. The techniques we use are straightforward, and robust enough to handle even the pathological case of nowhere-differentiable activation functions (Proposition \ref{nondiffthm}).


We also consider approximation in $L^p$ norm (Remark \ref{relucorrem}), and generalise previous work to smaller widths, multiple output neurons, and $p \geq 1$ in place of $p =1$ (Theorem \ref{relucor}).

In contrast to much previous work, every result we show also handles the general case of multiple output neurons.

\acks{This work was supported by the Engineering and Physical Sciences Research Council [EP/L015811/1].}

\bibliography{deep_approx}

\appendix

\section{Proof of the Register Model (Proposition \ref{register-model})}\label{register-model-appendix}

First, we recall the classical Universal Approximation Theorem \citep{pinkus}:
\Cybenko*

The Register Model is created by suitably reorganising the neurons from a collection of such shallow networks.

\registermodel*

\begin{proof}
Fix $f \in C(K; \reals^m)$. Let $f = (f_1, \ldots, f_m)$. Fix $\epsilon > 0$. By Theorem \ref{Cybenko}, there exist single-hidden-layer neural networks $g_1, \ldots, g_m \in \NNn$ with activation function $\rho$ approximating $f_1, \ldots, f_m$ respectively. Each approximation is to within error $\epsilon$ with respect to $\chebnorm{\,\cdot\,}$ on $K$. Let each $g_i$ have $\beta_i$ hidden neurons. Let $\sigma_{i, j}$ represent the operation of its $j$th hidden neuron, for $j \in \{1, \ldots, \beta_i\}$. In keeping with the idea of enhanced neurons, let each $\sigma_{i, j}$ include the affine function that comes after it in the output layer of $g_i$, so that $g_i = \sum_{j = 1}^{\beta_i} \sigma_{i, j}$. Let $M = \sum_{i = 1}^m \beta_i$.

We seek to construct a neural network $N \in \Inm$. Given input $(x_1, \ldots, x_n) \in \reals^n$, it will output $(G_1, \ldots, G_m) \in \reals^m$, such that $G_i = g_i(x_1, \ldots, x_n)$ for each $i$. That is, it will compute all of the shallow networks $g_1, \ldots, g_m$. Thus it will approximate $f$ to within error $\epsilon$ with respect to $\chebnorm{\,\cdot\,}$ on $K$.

The construction of $N$ is mostly easily expressed pictorially; see Figure \ref{table:main}. In each cell, representing a neuron, we define its value as a function of the values of the neurons in the previous layer. In every layer, all but one of the neurons uses the identity activation function $\iota \colon \reals \to \reals$, whilst one neuron in each layer performs a computation of the form $\sigma_{i, j}$.

The construction can be summed up as follows.

Each layer has $n + m + 1$ neurons, arranged into a group of $n$ neurons, a group of a single neuron, and a group of $m$ neurons.

The first $n$ neurons in each layer simply record the input $(x_1, \ldots, x_n)$, by applying an identity activation function. We refer to these as the `in-register neurons'.

Next we consider $g_1, \ldots, g_m$, which are all shallow networks. The neurons in the hidden layers of $g_1, \ldots, g_m$ are arranged `vertically' in our deep network, one in each layer. This is the neuron in each layer that uses the activation function $\rho$. We refer to these as the `computation neurons'. Each computation neuron performs its computation based off of the inputs preserved in the in-register neurons.

The final group of $m$ neurons also use the identity activation function; their affine parts gradually sum up the results of the computation neurons. We refer to these as the `out-register neurons'. The $i$th out-register neuron in each layer will sum up the results of the computation neurons computing $\sigma_{i, j}$ for all $j \in \{1, \ldots, \beta_i\}$.

Finally, the neurons in the output layer of the network are connected to the out-register neurons of the final hidden layer. As each of the neurons in the output layer has, as usual, the identity activation function, they will now have computed the desired results.
\end{proof}
\begin{figure}\centering
\begin{tikzpicture}
\node[rotate=90] at (0, 0)
{\begin{minipage}{0.98\textheight}\fontsize{7pt}{7pt}\selectfont
\begin{tabular}[t]{@{}c@{}c@{}ccc@{}c@{}c@{}c@{}c@{}c@{}ccc}
\llayer($G_1 = \zeta_{1, M}$\vphantom{0pt}, $G_2 = \zeta_{2, M}$, $G_m = \zeta_{m, M} + \tau_{m, M}$)
\toprule
\layer($\gamma_{1, M} = 0$, $\gamma_{2, M} = 0$, $\gamma_{n, M} = 0$, \vv{m}{\beta_m}{M - 1}{M}, $\zeta_{1, M} =$\\$ \iota(\zeta_{1, M - 1})$, $\zeta_{2, M} =$\\$ \iota(\zeta_{2, M - 1})$, $\zeta_{m, M} =$\\$ \iota(\zeta_{m, M - 1} + \tau_{M - 1}{)}$)
\layer(\y{1}{M - 1}{M - 2}, \y{2}{M - 1}{M - 2}, \y{n}{M - 1}{M - 2}, \vv{m}{\beta_m - 1}{M - 2}{M - 1}, $\zeta_{1, M - 1} =$\\$ \iota(\zeta_{1, M - 2})$, $\zeta_{2, {M - 1}} =$\\$ \iota(\zeta_{2, M - 2})$, $\zeta_{m, M - 1} =$\\$ \iota(\zeta_{m, M - 2} + \tau_{M - 2}{)}$)
\multicolumn{13}{c}{\rotatedcdots\vspace{2pt}}\\
%
\toprule
\multicolumn{13}{c}{\rotatedcdots\vspace{2pt}}\\
\toprule
\slayer(\y{1}{\beta_1 + \beta_2}{\beta_1 + \beta_2 - 1}, \y{2}{\beta_1 + \beta_2}{\beta_1 + \beta_2 - 1}, \y{n}{\beta_1 + \beta_2}{\beta_1 + \beta_2 - 1}, \uu{2}{\beta_2}{\beta_1 + \beta_2 - 1}, $\zeta_{1, \beta_1 + \beta_2} =$\\$ \iota(\zeta_{1, \beta_1 + \beta_2 - 1})$, $\zeta_{2, \beta_1 + \beta_2} =$\\$ \iota(\zeta_{2, \beta_1 + \beta_2 - 1} + \tau_{\beta_1 + \beta_2 - 1})$, $\zeta_{m, \beta_1 + \beta_1} = 0$)
\multicolumn{13}{c}{\rotatedcdots\vspace{2pt}}\\
\layer(\y{1}{\beta_1 + 3}{\beta_1 + 2}, \y{2}{\beta_1 + 3}{\beta_1 + 2}, \y{n}{\beta_1 + 3}{\beta_1 + 2}, \uuu{2}{3}{\beta_1 + 2}, $\zeta_{1, \beta_1 + 3} =$\\$ \iota(\zeta_{1, \beta_1 + 2})$, $\zeta_{2, {\beta_1 + 3}} =$\\$ \iota(\zeta_{2, \beta_1 + 2} + \tau_{\beta_1 + 2})$, $\zeta_{m, \beta_1 + 3} = 0$)
\layer(\y{1}{\beta_1 + 2}{\beta_1 + 1}, \y{2}{\beta_1 + 2}{\beta_1 + 1}, \y{n}{\beta_1 + 2}{\beta_1 + 1}, \uuu{2}{2}{\beta_1 + 1}, $\zeta_{1, \beta_1 + 2} =$\\$ \iota(\zeta_{1, \beta_1 + 1})$, $\zeta_{2, {\beta_1 + 2}} = \iota(\tau_{\beta_1 + 1})$, $\zeta_{m, \beta_1 + 2} = 0$)
\layer(\y{1}{\beta_1 + 1}{\beta_1}, \y{2}{\beta_1 + 1}{\beta_1}, \y{n}{\beta_1 + 1}{\beta_1}, \uuu{2}{1}{\beta_1}, $\zeta_{1, \beta_1 + 1} =$\\$ \iota(\zeta_{1, \beta_1} + \tau_{\beta_1})$, $\zeta_{2, \beta_1 + 1} = 0$, $\zeta_{m, \beta_1 + 1} = 0$)
\toprule
\slayer(\y{1}{\beta_1}{\beta_1 - 1}, \y{2}{\beta_1}{\beta_1 - 1}, \y{n}{\beta_1}{\beta_1 - 1}, \sss{1}{\beta_1}{\beta_1 - 1}, $\zeta_{1, \beta_1} =$\\$ \iota(\zeta_{1, \beta_1 - 1} + \tau_{\beta_1 - 1})$, $\zeta_{2, \beta_1} = 0$, $\zeta_{m, \beta_1} = 0$)
\multicolumn{13}{c}{\rotatedcdots\vspace{2pt}}\\
\layer(\y{1}{4}{3}, \y{2}{4}{3}, \y{n}{4}{3}, \sss{1}{4}{3}, $\zeta_{1, 4} =$\\$ \iota(\zeta_{1, 3} + \tau_3)$, $\zeta_{2, 4} = 0$, $\zeta_{m, 4} = 0$)
\layer(\y{1}{3}{2}, \y{2}{3}{2}, \y{n}{3}{2}, \sss{1}{3}{2}, $\zeta_{1, 3} =$\\$ \iota(\zeta_{1, 2} + \tau_2)$, $\zeta_{2, 3} = 0$, $\zeta_{m, 3} = 0$)
\layer(\y{1}{2}{1}, \y{2}{2}{1}, \y{n}{2}{1}, \sss{1}{2}{1}, $\zeta_{1, 2} = \iota(\tau_1)$, $\zeta_{2, 2} = 0$, $\zeta_{m, 2} = 0$)
\layer($\gamma_{1, 1} =$\\$ \iota(x_1)$, $\gamma_{2, 1} =$\\$ \iota(x_2)$, $\gamma_{n, 1} =$\\$ \iota(x_n)$, $\tau_1 = \sigma_{1, 1}(x_1{,} \ldots{,} x_n)$, \vspace{0pt}$\zeta_{1, 1} = 0$, $\zeta_{2,1} = 0$, $\zeta_{m,1}=0$)
\toprule
\flayer($x_1$, $x_2$, $x_n$)
\end{tabular}\normalsize
\caption{The thick lines delimit groups of layers; the $i$th group computes $\sigma_{i, 1}, \ldots, \sigma_{i, \beta_i}$. The inputs to the network are $x_1, \ldots, x_n$, depicted at the bottom. The outputs from the network are $G_1, \ldots, G_m$, depicted at the top. The identity activation function  $\reals \to \reals$ is denoted $\iota$.}\label{table:main}
\end{minipage}};
\end{tikzpicture}
\end{figure}
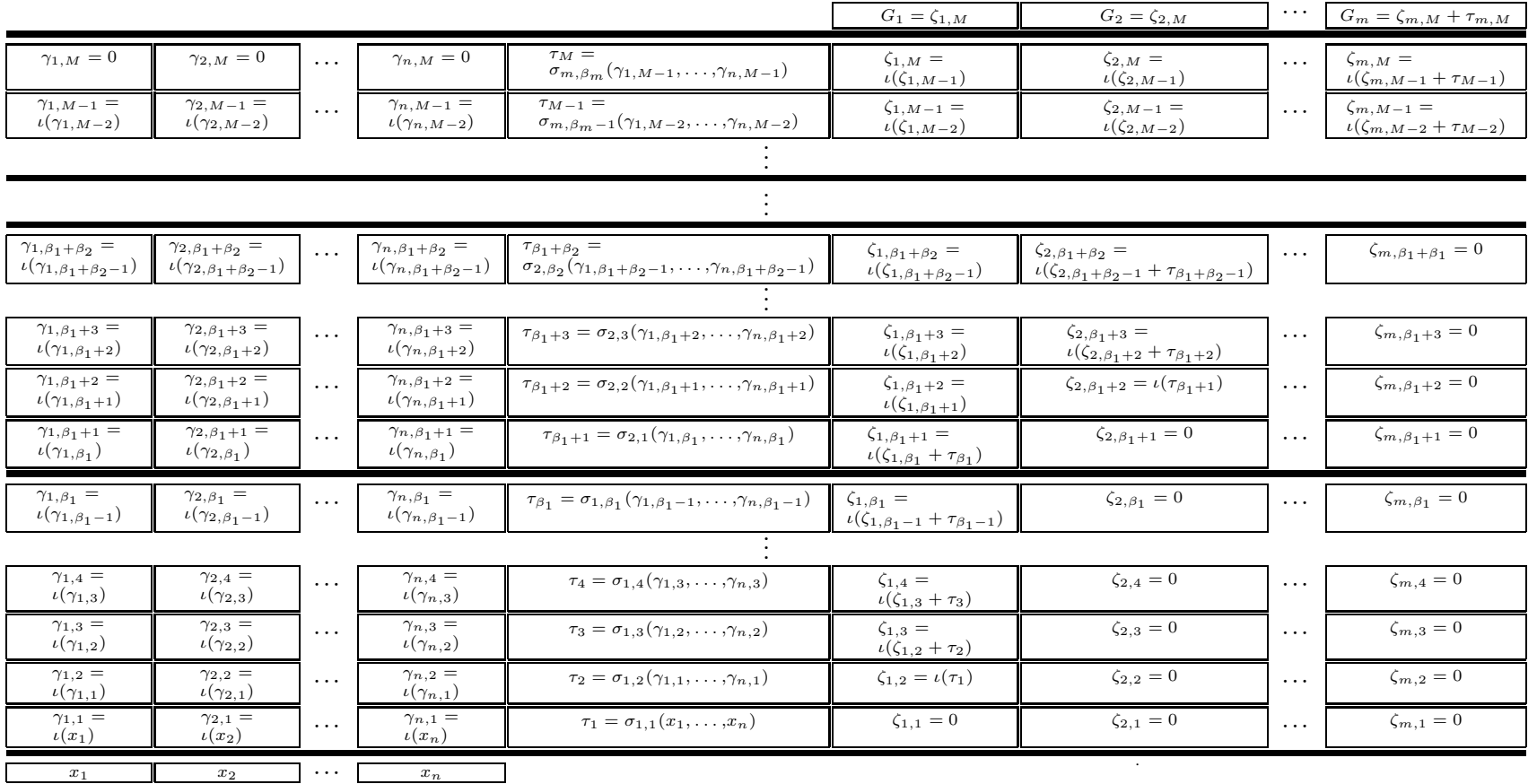
\newpage
\section{Proof of Theorem \ref{relucor}}\label{relucorproof}

\begin{lem}\label{urysohn}
Let $a, b, c, d \in \reals$ be such that $a < b < c < d$. Let $U_{a,b,c,d} \colon \reals \to \reals$ be the unique continuous piecewise affine function which is one on $[b, c]$ and zero on $(-\infty, a] \cup [d, \infty)$. Then two layers of two enhanced neurons each, with ReLU activation function, may exactly represent the function $U_{a,b,c,d}$.
\end{lem}
\begin{proof}
Let $x \in \reals$ be the input. Let $m_1 = 1/(b - a)$. Let $m_2 = 1/(d - c)$. Let $\eta_1, \eta_2$ represent the first neuron in each layer, and $\zeta_1, \zeta_2$ represent the second neuron in each layer. We assign them values as follows.
\begin{align*}
\eta_1 &= \max \{0, m_1 (x - a)\}, & \zeta_1 &= \max \{0, m_2 (x - c)\}, \\
\eta_2 &= \max \{0, 1 - \eta_1 \}, & \zeta_2 &=  \max \{0, 1 - \zeta_1 \}.
\end{align*}
Then $U_{a,b,c,d}(x) = \zeta_2 - \eta_2$.
\end{proof}

\begin{lem}
One layer of two enhanced neurons, with ReLU activation function, may exactly represent the function $(x, y) \mapsto \min \{x, y \}$ on $[0, \infty)^2$.\label{umin}  
\end{lem}
\begin{proof}
Let the first neuron compute $\eta = \max \{0, x - y\}$. Let the second neuron compute $\zeta = \max \{0, x\}$. Then $\min\{x,y\} = \zeta-\eta$.
\end{proof}

\relucor*
\begin{proof}
Let $f \in L^p(\reals^n; \reals^m)$ and $\epsilon > 0$. For simplicity assume that $\reals^m$ is endowed with the $\chebnorm{\,\cdot\,}$ norm; other norms are of course equivalent. Let $\widehat{f} = (\widehat{f}_1, \ldots, \widehat{f}_m) \in C_c(\reals^n; \reals^m)$ be such that
\begin{equation}\label{eq:hatfdef}
\norm{f - \widehat{f}}_p < \epsilon / 3.
\end{equation}
Let
\begin{equation}\label{eq:cmax}
C = \sup_{x \in \reals^n} \max_i \widehat{f}_i(x) + 1
\end{equation}
and
\begin{equation}\label{eq:cmin}
c = \inf_{x \in \reals^n} \min_i \widehat{f}_i(x) - 1
\end{equation}

Pick $a_1, b_1, \ldots, a_n, b_n \in \reals$ such that $J$ defined by
\begin{equation*}
J = [a_1, b_1] \times \cdots \times [a_n, b_n]
\end{equation*}
is such that $\supp \widehat{f} \subseteq J$. Furthermore, for $\delta > 0$ that we shall fix in a moment, let
\begin{align*}
A_i &= a_i - \delta, \\
B_i &= b_i + \delta,
\end{align*}
and let $K$ be defined by
\begin{equation*}
K = [A_1, B_1] \times \cdots \times [A_n, B_n].
\end{equation*}
Fix $\delta$ small enough that
\begin{equation}\label{eq:complicated}
\abs{K \setminus J}^{1/p} \cdot \max\left\{\abs{C}, \abs{c}\right\} < \frac{\epsilon}{6}.
\end{equation}

Let $g  = (g_1, \ldots, g_m) \in \NN{n + m + 1}$ be such that
\begin{equation}\label{eq:defg}
\sup_{x \in K}\abs{\widehat{f}(x) - g(x)} < \min\left\{\frac{\epsilon}{3 \abs{J}^{1/p}},\,1\right\},
\end{equation}
which exists by Proposition \ref{mainthm}. Note that $g$ is defined on all of $\reals^n$; it simply happens to be close to $\widehat{f}$ on $K$. In particular it will takes values close to zero on $K \setminus J$, and may take arbitrary values in $\reals^n \setminus K$. By equations (\ref{eq:cmax}), (\ref{eq:cmin}), (\ref{eq:defg}), it is the case that
\begin{align}
C \geq \sup_{x \in K} \max_i g_i(x),\nonumber\\
c \leq \inf_{x \in K} \min_i g_i(x).\label{eq:cminmax}
\end{align}

Now consider the network describing $g$; we will modify it slightly. The goal is to create a network which takes value $g$ on $J$, zero in $\reals^n \setminus K$, and moves between these values in the interface region $K \setminus J$. Such a network will provide a suitable approximation to $\widehat{f}$.

This will be done by first constructing a function which is approximately the indicator function for $J$, with support in $K$; call such a function $U$. The idea then is to construct a neural representation of $G_i$ defined by
\begin{equation*}
G_i = \min\{\max\{g_i, cU\}, CU\}.
\end{equation*}
Provided $\abs{K \setminus J}$ is small enough then $G = (G_1, \ldots, G_m)$ will be the desired approximation; this is proved this below.

We move on to presenting the neural representation of this construction.

First we observe that because the activation function is the ReLU, then the identity approximations used in the proof of Proposition \ref{mainthm} may in fact exactly represent the identity function on some compact set: $x \mapsto \max\{0, x+N\} - N$ is exactly the identity function, for suitably large $N$, and is of the form that an enhanced neuron may represent. This observation isn't strictly necessary for the proof, but it does simplify the presentation somewhat, as the values preserved in the in-register neurons of $g$ are now exactly the inputs $x = (x_1, \ldots, x_n)$ for $x \in K$. For sufficiently negative $x_i$, outside of $K$, they will take the value $-N$ instead, but by insisting that is $N$ sufficiently large that
\begin{equation}\label{eq:N_below_A}
-N < \min_{i \in \{1, \ldots, n\}} A_i,
\end{equation}
then this will not be an issue for the proof.

So take the network representing $g$, and remove the output layer. (If the output layer is performing any affine transformations then treat them as being part of the final hidden layer, in the manner of enhanced neurons. Thus the output layer that is being removed is just applying the identity function to the out-register neurons.) Some more hidden layers will be placed on top, and then a new output layer will be placed on top. In the following description, all neurons not otherwise specified will be performing the identity function, so as to preserve the values of the corresponding neurons in the preceding layer. As all functions involved are continuous and $K$ is compact, and compactness is preserved by continuous functions, and continuous functions are bounded on compact sets, then this is possible for all $x \in K$ by taking $N$ large enough.

The first task is to modify the value stored in the $x_1$-in-register neuron. At present it stores the value $x_1$; by using the $x_1$-in-register neuron and the computation neuron in two extra layers, its value may be replaced with $U_{A_1,a_1,b_1,B_1}(x_1)$, via Lemma \ref{urysohn}. Place another two layers on top, and use them to replace the value of $x_2$ in the $x_2$-in-register neuron with $U_{A_2,a_2,b_2,B_2}(x_2)$, and so on. The in-register neurons now store the values
\begin{equation*}
U_{A_1,a_1,b_1,B_1}(x_1), \ldots, U_{A_n,a_n,b_n,B_n}(x_n).
\end{equation*}

Once this is complete, place another layer on top and use the first $x_1$-in-register neuron and the $x_2$-in-register neuron to compute the minimum of their values, in the manner of Lemma \ref{umin}, thus computing
\begin{equation*}
\min\{U_{A_1,a_1,b_1,B_1}(x_1), U_{A_2,a_2,b_2,B_2}(x_2)\}.
\end{equation*}
Place another layer on top and use another two in-register neurons to compute the minimum of this value and the value presently stored in the $x_3$-in-register neuron, that is $U_{A_3,a_3,b_3,B_3}(x_3)$, so that
\begin{equation*}
\min\{U_{A_1,a_1,b_1,B_1}(x_1), U_{A_2,a_2,b_2,B_2}(x_2), U_{A_3,a_3,b_3,B_3}(x_3)\}
\end{equation*}
has now been computed. Continue to repeat this process until the in-register neurons have computed\footnote{It doesn't matter which of the in-register neurons records the value of $U$.}
\begin{equation*}
U = \min_{i \in \{1, \ldots, n\}} U_{A_i,a_i,b_i,B_i}(x_i).
\end{equation*}
Observe how $U$ represents an approximation to the indicator function for $J$, with support in $K$, evaluated at $(x_1, \ldots, x_n)$.

(Note how the small foible regarding how an in-register neuron would only record $-N$ instead of $x_i$, for $x_i < -N$, is not an issue. This is because of equation (\ref{eq:N_below_A}), which implies that $U_{A_i, a_i, b_i, B_i}(x_i) = 0 = U_{A_i, a_i, b_i, B_i}(-N)$, thus leaving the value of $U$ unaffected.)

This is a highly destructive set of operations: the network no longer remembers the values of its inputs. Thankfully, it no longer needs them.

The out-register neurons presently store the values $g_1, \ldots, g_m$, where $g_i = g_i(x_1, \ldots, x_n)$. Now add another layer. Let the value of its out-register neurons be $\theta_1, \ldots, \theta_m$, where
\begin{equation*}
\theta_i = \max\{0, g_i - c U \}.
\end{equation*}
Add one more hidden layer. Let the value of its out-register neurons be $\lambda_1, \ldots, \lambda_m$, where
\begin{equation*}
\lambda_i = \max\{0, -\theta_i + (C - c) U\}.
\end{equation*}
Finally place the output layer on top. Let the value of its neurons be $G_1, \ldots, G_m$, where
\begin{equation*}
G_i = -\lambda_i + C U.
\end{equation*}
Then in fact
\begin{equation}\label{eq:Gfromg}
G_i = \min\{\max\{g_i, cU\}, CU\}
\end{equation}
as desired.

All that remains to show is that $G = (G_1, \ldots, G_m)$ of this form is indeed a suitable approximation. First, as $G$ and $g$ coincide in $J$, and by equation (\ref{eq:defg}),
\begin{align}
\left(\int_J \abs{\widehat{f}(x) - G(x)}^p \dee x \right)^{1/p} &\leq \abs{J}^{1/p} \sup_{x \in J}\abs{\widehat{f}(x) - G(x)} \nonumber\\
&= \abs{J}^{1/p} \sup_{x \in J}\abs{\widehat{f}(x) - g(x)}\nonumber\\
&< \frac{\epsilon}{3}.\label{eq:diffone}
\end{align}
Secondly, by equations (\ref{eq:cmax}), (\ref{eq:cmin}), (\ref{eq:cminmax}), (\ref{eq:Gfromg}) and then equation (\ref{eq:complicated}),
\begin{align}
\left(\int_{K \setminus J} \abs{\widehat{f}(x) - G(x)}^p \dee x \right)^{1/p} &\leq \abs{K \setminus J}^{1/p} \sup_{x \in K \setminus J}\abs{\widehat{f}(x) - G(x)} \nonumber\\
&\leq \abs{K \setminus J}^{1/p} \cdot 2\max\left\{\abs{C}, \abs{c}\right\}\nonumber\\
&< \frac{\epsilon}{3}.\label{eq:difftwo}
\end{align}
Thirdly,
\begin{align}
\left(\int_{\reals^n \setminus K} \abs{\widehat{f}(x) - G(x)}^p \dee x \right)^{1/p} = 0, \label{eq:diffthree}
\end{align}
as both $\widehat{f}$ and $G$ have support in $K$.

So by equations (\ref{eq:hatfdef}), (\ref{eq:diffone}), (\ref{eq:difftwo}) and (\ref{eq:diffthree}),
\begin{equation*}
\norm{f - G}_p \leq \norm{f - \widehat{f}}_p + \norm{\widehat{f} - G}_p < \epsilon
\end{equation*}
\end{proof}

\section{Proof of Proposition \ref{polythm}}\label{polythmproof}
\polythm*
\begin{proof}
Let $h \in \reals \setminus \{0\}$. Define $\rho_h \colon \reals \to \reals$ by
\begin{equation*}
\rho_h(x) = \frac{\rho(\alpha + hx) - \rho(\alpha)}{h^2 \rho''(\alpha)/2}.
\end{equation*}
Then, taking a Taylor explansion around $\alpha$,
\begin{align*}
\rho_h(x) &= \frac{\rho(\alpha)  + hx\rho'(\alpha) +h^2x^2\rho''(\alpha)/2 + \bigO(h^3 x^3) - \rho(\alpha)}{h^2 \rho''(\alpha)/2}\\
&= x^2 + \bigO (h x^3).
\end{align*}
Let $s(x) = x^2$. Then $\rho_h \to s$ uniformly over any compact set as $h \to 0$.

Now set up a network as in the Square Model (Proposition \ref{squareprop}), with every neuron using the square activation function. Call this network $N$. Create a network $N_h$ by copying $N$ and giving every neuron in the network the activation function $\rho_h$ instead.

Uniform continuity preserves uniform convergence, compactness is preserved by continuous functions, and a composition of two uniformly convergent sequences of functions with uniformly continuous limits is again uniformly convergent. 
Thus the difference between $N$ and $N_h$, with respect to $\chebnorm{\,\cdot\,}$ on $K$, may be taken arbitrarily small by taking $h$ arbitrarily small. 

Furthermore note that $\rho_h$ is just $\rho$ pre- and post-composed with affine functions. (Note that there is only one term in the definition of $\rho_h(x)$ which depends on $x$.) This means that any network which may be represented with activation function $\rho_h$ may be precisely represented with activation function $\rho$, by combining the affine transformations involved.
\end{proof}

\end{document}